\documentclass[draft]{tlp}

\newcommand\bcmdtab{\noindent\bgroup\tabcolsep=0pt%
  \begin{tabular}{@{}p{10pc}@{}p{20pc}@{}}}
\newcommand\ecmdtab{\end{tabular}\egroup}

  \title[Theory and Practice of Logic Programming]
        {Complexity and Compilation of GZ-Aggregates in Answer Set Programming}

  \author[M. Alviano and N. Leone]
         {MARIO ALVIANO and NICOLA LEONE\\
         Department of Mathematics and Computer Science, University of Calabria, Italy}

\jdate{March 2003}
\pubyear{2003}
\pagerange{\pageref{firstpage}--\pageref{lastpage}}
\doi{XXX}

\usepackage{amsmath}
\usepackage{amssymb} 
\usepackage{comment}
\usepackage{xspace}
\usepackage{mathtools}
\usepackage{url}
\usepackage{thmtools}
\usepackage{thm-restate}
\usepackage[obeyDraft]{todonotes}
\usepackage{booktabs}

\newtheorem{definition}{Definition}
\newtheorem{example}{Example}

\newcommand{\A}{\ensuremath{\mathcal{U}}\xspace}
\newcommand{\T}{\ensuremath{\mathbf{T}}\xspace}
\newcommand{\F}{\ensuremath{\mathbf{F}}\xspace}
\newcommand{\naf}{\ensuremath{\raise.17ex\hbox{\ensuremath{\scriptstyle\mathtt{\sim}}}}\xspace}
\newcommand{\COUNT}{\ensuremath{\textsc{count}}\xspace}
\newcommand{\SUM}{\ensuremath{\textsc{sum}}\xspace}
\newcommand{\AVG}{\ensuremath{\textsc{avg}}\xspace}
\newcommand{\MIN}{\ensuremath{\textsc{min}}\xspace}
\newcommand{\MAX}{\ensuremath{\textsc{max}}\xspace}
\newcommand{\ODD}{\ensuremath{\textsc{odd}}\xspace}
\newcommand{\EVEN}{\ensuremath{\textsc{even}}\xspace}
\DeclarePairedDelimiter\norm{\lVert}{\rVert}

\begin{document}

\label{firstpage}

\maketitle

\begin{abstract}
Gelfond and Zhang recently proposed a new stable model semantics based on Vicious Circle Principle in order to improve the interpretation of logic programs with aggregates.
The paper focuses on this proposal, and analyzes the complexity of both coherence testing and cautious reasoning under the new semantics.
Some surprising results highlight similarities and differences versus mainstream stable model semantics for aggregates.
Moreover, the paper reports on the design of compilation techniques for implementing the new semantics on top of existing ASP solvers, which eventually lead to realize a prototype system that allows for experimenting with Gelfond-Zhang's aggregates.
\end{abstract}



\section{Introduction}

Answer set programming (ASP) is a declarative language for knowledge representation and reasoning \cite{DBLP:journals/cacm/BrewkaET11}.
ASP specifications are sets of logic rules, possibly using disjunction and default negation, interpreted according to the stable model semantics \cite{DBLP:conf/iclp/GelfondL88,DBLP:journals/ngc/GelfondL91}.
The basic language is extended by several constructs to ease the representation of practical knowledge.
Aggregate functions are among these extensions \cite{DBLP:journals/ai/SimonsNS02,DBLP:journals/ai/LiuPST10,DBLP:journals/ai/FaberPL11,DBLP:conf/aaaiss/BartholomewLM11}, and allow to express properties on sets of atoms declaratively.
For example, aggregate functions are often used to enforce \emph{functional dependencies};
a rule of the following form:
\begin{equation*}
 \bot \leftarrow R'(\overline{X}),\ \COUNT[\overline{Y} : R(\overline{X},\overline{Y},\overline{Z})] \leq 1
\end{equation*}
constrains relation $R$ to satisfy the functional dependency $\overline{X} \rightarrow \overline{Y}$, where $\overline{X} \cup \overline{Y} \cup \overline{Z}$ is the set of attributes of $R$, and $R'$ is the projection of $R$ on $\overline{X}$.
Aggregate functions are also commonly used in ASP to constrain a nondeterministic guess.
For example, in the \emph{knapsack problem} the total weight of the selected items must not exceed a given limit, which can be modeled by the following rule aggregating over a multiset:
\begin{equation*}
 \bot \leftarrow \SUM[W,O : \mathit{object}(O,W,C), \mathit{in}(O)] \leq \mathit{limit}.
\end{equation*}

Aggregate functions may also ease the representation of logic circuits made of gates of unbounded fan-in \cite{DBLP:journals/tplp/GelfondZ14};
the following rule models that the output of an XOR gate is 1 if an odd number of its inputs have value 1:
\begin{equation*}
  \mathit{value}(O,1) \leftarrow \mathit{xor}(G), \mathit{output}(G,O),
  \ODD[I : \mathit{input}(G,I), \mathit{value}(I,1)].
\end{equation*}

Several semantics were proposed for ASP programs with aggregates.
Two of them \cite{DBLP:journals/tocl/Ferraris11,DBLP:journals/ai/FaberPL11} are implemented in popular ASP solvers \cite{DBLP:journals/ai/GebserKS12,DBLP:journals/tplp/FaberPLDI08}.
These two semantics agree for programs without negated aggregates, and are referred in this paper as F-stable model semantics.
An alternative semantics, presented at the 30th International Conference on Logic Programming (ICLP'14) by \cite{DBLP:journals/tplp/GelfondZ14} and here referred to as GZ- or G-stable model semantics, is based on the notion of \emph{vicious circle principle}, which essentially asserts that the truth of an atom must be inferred by means of a definition not referring, directly or indirectly, to the truth of the atom itself.

The present paper explores this new semantics, reporting a detailed complexity analysis of \emph{coherence testing} and \emph{cautious reasoning} \cite{DBLP:journals/amai/EiterG95}, two of the main computational tasks in ASP.
In a nutshell, coherence testing amounts to check the existence of a stable model of an input program, while cautious reasoning consists in checking whether a given atom is true in all stable models of a program.

Concerning coherence testing, membership in $\Sigma^P_2$ was proved in \cite{DBLP:journals/tplp/GelfondZ14}, and $\Sigma^P_2$-hardness is proved here already for negation-free programs with a very limited form of aggregate functions, referred to as \emph{monotone} aggregates in the literature.
This result is in contrast with F-stable model semantics, for which coherence of negation-free programs with monotone aggregates is guaranteed.
Whether this must be considered a strength or a weakness of G-stable models is not the focus of this paper, but we remark here that the increase in complexity also comes with a higher expressive power in this case:
aggregates referred to as monotone in the literature allow to simulate integrity constraints and possibly default negation when interpreted according to the semantics by \cite{DBLP:journals/tplp/GelfondZ14}.
Moreover, there are also many cases in which G-stable models actually decrease the complexity of the reasoning tasks.
In fact, while for F-stable model semantics coherence testing is $\Sigma^P_2$-hard already for disjunction-free programs, this computational task is proved to be NP-complete for these programs under G-stable model semantics.
Finally, P-completeness is proved for programs with monotone aggregates if disjunction and negation are not used, a result compatible with F-stable model semantics.
However, also in this case G-stable models allow to simulate integrity constraints, which is not possible with F-stable models.

As for the complexity of cautious reasoning, membership and hardness in the complementary complexity classes are proved for all the analyzed fragments of the language.
These complexity results also implicitly characterize the computational complexity of \emph{brave reasoning}, another common reasoning task in ASP which consists in checking whether a given propositional atom is true in some stable model of an input program.
In fact, brave reasoning has the same complexity of coherence testing under G-stable model semantics, while this is not necessarily the case for F-stable models.
Again, the reason for this discrepancy is the power of G-stable models to simulate integrity constraints, as this is the additional construct that is commonly used for reducing brave reasoning to coherence testing.

Further results in the paper are two rewriting techniques for compiling programs interpreted according to G-stable semantics into programs interpreted according to F-stable semantics.
The first rewriting is simpler and introduces fewer auxiliary symbols, while the second has the advantage of producing programs with \emph{non recursive} aggregates only.
Both rewritings are \emph{polynomial, faithful and modular} translation functions \cite{DBLP:journals/jancl/Janhunen06}, and are implemented in a system prototype.
It is publicly available (\url{http://alviano.net/software/g-stable-models/}) and allows for experimenting with this newly proposed semantics.

\section{Background}\label{sec:background}

After defining the syntax of logic programs with aggregates, two semantics are introduced, referred to as F- \cite{DBLP:journals/tocl/Ferraris11,DBLP:journals/ai/FaberPL11} and G-stable models \cite{DBLP:journals/tplp/GelfondZ14}.
It is remarked here, and clarified in Section~\ref{sec:related}, that the original definitions are properly adapted to better fit the results in this paper.

\paragraph{Syntax.}
Let \T,\F denote the Boolean truth values true and false, respectively.
Let \A be a finite set of propositional atoms.
An aggregate atom $A$ is a Boolean function whose domain, denoted $\mathit{dom}(A)$, is a subset of \A.
A literal is a propositional atom, or a propositional atom preceded by (one or more occurrences of) the negation as failure symbol \naf, or an aggregate atom.
A rule $r$ is of the following form:
\begin{equation}\label{eq:rule}
 p_1 \vee \cdots \vee p_m \leftarrow l_1, \ldots, l_n
\end{equation}
where $p_1,\ldots,p_m$ are propositional atoms, $l_1,\ldots,l_n$ are literals, $m \geq 1$ and $n \geq 0$.
Set $\{p_1,\ldots,p_m\}$ is the head of $r$, denoted $H(r)$, and set $\{l_1,\ldots,l_n\}$ is the body of $r$, denoted $B(r)$.
A program $\Pi$ is a finite set of rules of the form (\ref{eq:rule}).
The set of propositional atoms occurring in $\Pi$ is denoted $At(\Pi)$.

\paragraph{Semantic notions.}
An interpretation $I$ is a subset of \A.
Let $S,S'$ be sets of interpretations, and $C$ be a set of propositional atoms.
Sets $S$ and $S'$ are equivalent in the context $C$, denoted $S \equiv_C S'$, if $|S| = |S'|$ and $\{I \cap C \mid I \in S\} = \{I \cap C \mid I \in S'\}$.
Aggregates are usually classified in three groups \cite{DBLP:journals/jair/LiuT06}:
an aggregate $A$ is \emph{monotone} if $A(I) = \T$ implies $A(J) = \T$, for all $I \subseteq J \subseteq \A$;
an aggregate $A$ is \emph{convex} if $A(I) = A(K) = \T$ implies $A(J) = \T$, for all $I \subseteq J \subseteq K \subseteq \A$;
the remaining aggregates are called \emph{non-convex}.
Note that monotone aggregates are convex, and the inclusion is strict.
Relation $\models$ is inductively defined as follows:
for a propositional atom $p \in \A$, $I \models p$ if $p \in I$;
for a negated literal $\naf l$, $I \models \naf l$ if $I \not\models l$;
for an aggregate atom $A$, $I \models A$ if $A(I \cap \mathit{dom}(A)) = \T$;
for a set or conjunction $C$, $I \models C$ if $I \models p$ holds for each $p \in C$;
for a rule $r$, $I \models r$ if $H(r) \cap I \neq \emptyset$ whenever $I \models B(r)$.
$I$ is a \emph{model} of a program $\Pi$ if $I \models \Pi$, i.e., if $I \models r$ for all $r \in \Pi$.

\begin{example}\label{ex:syntax}\label{ex:model}
Let $I$ be an interpretation, and $k \geq 1$.
An aggregate $A$ such that $A(I)$ equals $|\mathit{dom}(A) \cap I| \geq k$ is monotone.
An aggregate $A$ such that $A(I)$ equals $|\mathit{dom}(A) \cap I| = k$ is convex.
An aggregate $A$ such that $A(I)$ equals $|\mathit{dom}(A) \cap I| \neq k$ is non-convex.

Let $A_1$ be an aggregate such that $\mathit{dom}(A_1) = \{a,b\}$ and $A_1(I)$ equals $|\{a,b\} \cap I| \geq 1$, for all interpretations $I$.
A program using $A_1$ is 
$\Pi_1 = \{a \leftarrow \naf\naf a;\ b \vee c \leftarrow A_1\}$.
The models of $\Pi_1$, restricted to the atoms occurring in the program, are the following:
$\emptyset$, $\{a,b\}$, $\{a,c\}$, $\{a,b,c\}$, $\{b\}$, $\{b,c\}$, and $\{c\}$.
\hfill $\blacksquare$
\end{example}

\paragraph{F-stable models.}
Let $\Pi$ be a program and $I$ an interpretation.
The F-reduct of $\Pi$ with respect to $I$ is defined as follows:
$F(\Pi,I) = \{F(r,I) \mid r \in \Pi,\ I \models B(r)\}$, where
$F(r,I) = p_1 \vee \cdots \vee p_m \leftarrow F(l_1,I), \ldots, F(l_n,I)$ for $r$ being of the form (\ref{eq:rule}),
$F(l,I) = l$ if $l$ is a propositional atom or an aggregate atom $A$, and
$F(l,I) = \emptyset$ if $l$ is a negative literal.
$I$ is an F-stable model of $\Pi$ if $I \models \Pi$ and there is no $J \subset I$ such that $J \models F(\Pi,I)$.
The set of F-stable models of $\Pi$ is denoted $\mathit{FSM}(\Pi)$.

\paragraph{G-stable models.}
Let $\Pi$ be a program and $I$ an interpretation.
The G-reduct of $\Pi$ with respect to $I$ is defined as follows:
$G(\Pi,I) = \{G(r,I) \mid r \in \Pi,\ I \models B(r)\}$, where
$G(r,I) = p_1 \vee \cdots \vee p_m \leftarrow G(l_1,I), \ldots, G(l_n,I)$ for $r$ being of the form (\ref{eq:rule}),
$G(l,I) = p$ if $l$ is a propositional atom $p$,
$G(l,I) = I \cap \mathit{dom}(A)$ if $l$ is an aggregate atom $A$, and
$G(l,I) = \emptyset$ if $l$ is a negative literal.
$I$ is a G-stable model of $\Pi$ if $I \models \Pi$ and there is no $J \subset I$ such that $J \models G(\Pi,I)$.
The set of G-stable models of $\Pi$ is denoted $\mathit{GSM}(\Pi)$.

\begin{example}\label{ex:flp}\label{ex:gz}
The F-stable models of $\Pi_1$ in Example~\ref{ex:syntax} are the following:
$\emptyset$, $\{a,b\}$, and $\{a,c\}$.
Indeed, note that $F(\Pi_1,\emptyset) = \emptyset$,
$F(\Pi_1,\{a,b\}) = F(\Pi_1,\{a,c\}) = \{a \leftarrow;\ b \vee c \leftarrow A_1\}$,
and each model is minimal for its reduct.
On the other hand, $\{b\}$ is not an F-stable model because $\emptyset$ is a model of
$F(\Pi_1,\{b\}) = \{b \vee c \leftarrow A_1\}$.
The G-stable models of $\Pi_1$ are the following:
$\emptyset$ and $\{a,c\}$.
Indeed, $G(\Pi_1,\emptyset) = \emptyset$ and
$G(\Pi_1,\{a,c\}) = \{a \leftarrow;\ b \vee c \leftarrow a\}$.
Note that $A_1$ is replaced by $a$ in the last rule of $G(\Pi_1,\{a,c\})$ because $\{a,c\} \cap \mathit{dom}(A_1) = \{a\}$.
Also observe that $\{a,b\}$ is not a G-stable model because 
$G(\Pi_1,\{a,b\}) = \{a \leftarrow;\ b \vee c \leftarrow a, b\}$,
and $\{a\}$ is a model of this reduct.
\hfill $\blacksquare$
\end{example}

\paragraph{Computational problems.}
Let $X \in \{F,G\}$.
A program $\Pi$ is \emph{$X$-coherent} if $\Pi$ has at least one $X$-stable model;
otherwise, $\Pi$ is \emph{$X$-incoherent}.
\emph{$X$-coherence testing} is the computational problem of checking whether an input program $\Pi$ is $X$-coherent.
A propositional atom $p$ is an \emph{$X$-cautious consequence} of $\Pi$ if $p$ belongs to all $X$-stable models of $\Pi$.
\emph{$X$-cautious reasoning} is the computational problem of checking whether a given atom $p$ is an $X$-cautious consequence of an input program $\Pi$.

\section{Complexity}\label{sec:complexity}

Complexity of F-cautious reasoning, and implicitly also of F-coherence testing, was analyzed in \cite{DBLP:journals/ai/FaberPL11}. 
A similar analysis is reported here for G-stable semantics, and in particular the combination of monotone (M), convex (C) and non-convex (N) aggregates with negation (\naf) and disjunction ($\vee$) is analyzed.
A summary of results is shown in Table~\ref{tab:complexity}, where all complexity bounds are tight.
Note that in some cases the existence of a stable model is guaranteed, and hence constant complexity $\mathbf{K}$ is reported.
Throughout this section, aggregates are assumed to be polynomial-time computable functions, and ASP($X$) will denote the class of programs using constructs in the list $X$.
For example, ASP($\neg$, C) is the class of programs possibly using negation and convex aggregates, while ASP($-$) is the class of programs not using negation, disjunction or aggregates.

\begin{table}[b]
 \caption{
    Complexity of G-coherence testing and G-cautious reasoning.
    An $\uparrow$ denotes an increase in complexity with respect to F-stable model semantics, where the considered complexity classes are $\mathbf{K} \subseteq \text{P} \subseteq \text{NP} \subseteq \Sigma^P_2$, and $\mathbf{K} \subseteq \text{P} \subseteq \text{co-NP} \subseteq \Pi^P_2$.
    Similarly, $\downarrow$ denotes a decrease in complexity.
 }\label{tab:complexity}
 \begin{center}
 \begin{tabular}{rrrrrrrrrr}
    \toprule
       & \multicolumn{4}{c}{\textsc{coherence testing}}                                                  & \multicolumn{4}{c}{\textsc{cautious reasoning}} \\
   \cmidrule{2-5}\cmidrule{6-9}
       & $\{\}$ \phantom{$\uparrow$}   & $\{\naf\}$ \phantom{$\uparrow$} & $\{\vee\}$ \phantom{$\uparrow\uparrow\uparrow$} & $\{\naf,\vee\}$                       & $\{\}$  \phantom{$\uparrow$}   & $\{\naf\}$  \phantom{$\uparrow$} & $\{\vee\}$ \phantom{$\uparrow$} & $\{\naf,\vee\}$ \\
   \cmidrule{1-9}
   --- & $\mathbf{K}$ \phantom{$\uparrow$} & NP \phantom{$\uparrow$} & $\mathbf{K}$ \phantom{$\uparrow\uparrow\uparrow$} & $\Sigma^P_2$ \phantom{l}           & P    \phantom{$\uparrow$}      & co-NP    \phantom{$\uparrow$}   & co-NP    \phantom{$\uparrow$}   & $\Pi^P_2$   \phantom{l}           \\
   M   & P $\uparrow$        & NP \phantom{$\uparrow$} & $\Sigma^P_2$ $\uparrow\uparrow\uparrow$     & $\Sigma^P_2$ \phantom{l}                             & P      \phantom{$\uparrow$}    & co-NP    \phantom{$\uparrow$}   & \phantom{$\uparrow$} $\Pi^P_2$ $\uparrow$ & $\Pi^P_2$    \phantom{l}           \\
   C   & NP \phantom{$\uparrow$} & NP \phantom{$\uparrow$}  & $\Sigma^P_2$ \phantom{$\uparrow\uparrow\uparrow$} & $\Sigma^P_2$ \phantom{l}                  & co-NP  \phantom{$\uparrow$}    & co-NP    \phantom{$\uparrow$}   & $\Pi^P_2$   \phantom{$\uparrow$}    & $\Pi^P_2$    \phantom{l}          \\
   N & NP $\downarrow$ & \phantom{$\downarrow$} NP $\downarrow$  & $\Sigma^P_2$ \phantom{$\uparrow\uparrow\uparrow$} & $\Sigma^P_2$ \phantom{l}         & co-NP $\downarrow$ & \phantom{$\downarrow$} co-NP $\downarrow$  & $\Pi^P_2$  \phantom{$\uparrow$} & $\Pi^P_2$  \phantom{l}           \\
   \bottomrule
 \end{tabular}
 \end{center}
\end{table}

\subsection{Complexity of Coherence Testing}

Complexity of coherence testing for programs without aggregates, reported in the first row of Table~\ref{tab:complexity}, is well-known (see for example \cite{DBLP:journals/amai/EiterG95}).
Membership in $\Sigma^P_2$ is implicit in \cite{DBLP:journals/tplp/GelfondZ14} for the general case.
For the other membership results, the immediate consequence operator is used.

\begin{definition}\label{def:tp}
Let $\Pi$ be a program and $I$ an interpretation.
The immediate consequence operator $T_\Pi$ is defined as follows:
$T_\Pi(I) = \{p \in H(r) \mid r \in \Pi, I \models B(r)\}$.
\end{definition}

For an ASP(M) program $\Pi$, $T_\Pi$ is monotone and therefore has a least fixpoint, and this fixpoint is computable in polynomial-time because a single application requires linear-time, and at most $|At(\Pi)|$ applications are required to reach the fixpoint.
Moreover, G-stable models of ASP(M) programs can be characterized in terms of $T_\Pi$, from which P-membership follows.

\begin{restatable}{lemma}{LemTpPoly}\label{lem:tp-poly}\label{lem:tp-sem}
Let $\Pi$ be in ASP(M).
The least fixpoint of $T_\Pi(I)$ exists and is polytime computable.
Let $I$ be the least fixpoint of $T_\Pi$, and $J$ be the least fixpoint of $T_{G(\Pi,I)}$.
If $I \neq J$ then $\Pi$ is G-incoherent, otherwise $\mathit{GSM}(\Pi) = \{I\}$.
\end{restatable}

\begin{restatable}{theorem}{ThmGcoDefP}\label{thm:gco-def-p}
G-coherence testing is in P for ASP(M).
\end{restatable}

To obtain NP-membership in the disjunction-free case, the following algorithm is used:
Guess a model $I$ of $\Pi$ and check that $I$ is a minimal model of $G(\Pi,I)$.
Checking that $I$ is a model of $\Pi$ and that $I$ is minimal for $G(\Pi,I)$ is polynomial-time doable (note that $G(\Pi,I)$ is in ASP($-$) and hence Lemma~\ref{lem:tp-poly} can be used).

\begin{restatable}{theorem}{ThmGcoNormalNp}\label{thm:gco-normal-np}
G-coherence testing is in NP for programs in ASP($\neg$, M, C, N).
\end{restatable}

As for the hardness, it is known that coherence testing is NP-hard if negation is present \cite{DBLP:journals/csur/DantsinEGV01}, while adding also disjunction increases the hardness to $\Sigma^P_2$ \cite{DBLP:journals/amai/EiterG95}.
These results propagate top-down in Table~\ref{tab:complexity}.
The missing results for programs with only convex or only non-convex aggregates can be obtained by the following transformations from aggregate-free programs with negation to negation-free programs with aggregates.

\begin{definition}\label{def:simple-rewritings}
Let $\Pi$ be in ASP($\neg$, $\vee$).
Let $C(\Pi)$ be the program obtained from $\Pi$ by replacing every occurrence of a negative literal $\naf p$ with an aggregate $A$ such that $\mathit{dom}(A) = \{p\}$ and $A(I) = |\{p\} \cap I| \leq 0$, for all $I \subseteq \A$.
Let $N(\Pi)$ be the program obtained from $\Pi$ by replacing every occurrence of a negative literal $\naf p$ with an aggregate $A$ such that $\mathit{dom}(A) = \{p,\bot\}$ and $A(I) = |\{p,\bot\} \cap I| \neq 1$, for all $I \subseteq \A$, where $\bot$ is a fixed atom not occurring in $\Pi$.
\end{definition}

\begin{restatable}{lemma}{LemSimpleRewritings}\label{lem:simple-rewritings}
Let $\Pi$ be in ASP($\neg$, $\vee$).
Then, 
$\mathit{GSM}(\Pi) \equiv_{At(\Pi)} \mathit{GSM}(C(\Pi)) \equiv_{At(\Pi)} \mathit{GSM}(N(\Pi))$.
\end{restatable}

Since $C(\Pi)$ and $N(\Pi)$ can be obtained in polynomial-time and only comprise convex and non-convex aggregates, respectively, the hardness results are obtained.

\begin{restatable}{theorem}{ThmHardOne}\label{thm:hard:1}
G-coherence testing is $\Sigma^P_2$-hard for both ASP($\vee$, C) and ASP($\vee$, N).
It is NP-hard for both ASP(C) and ASP(N).
\end{restatable}

The only missing cases are now for programs with monotone aggregates.
If disjunction and negation are not allowed, G-coherence testing is P-hard because of the following reduction:
Let $\Pi$ be in ASP($-$), and $p$ be a propositional atom.
Checking whether $p$ is a cautious consequence of $\Pi$ is equivalent to test coherence of $\Pi \cup \{p \leftarrow A\}$, where $\mathit{dom}(A) = \{p\}$ and $A(I) = |\{p\} \cap I| \geq 0$, for all $I \subseteq \A$.
Since cautious reasoning is P-hard for ASP($-$) (i.e., checking if a propositional atom belongs to the unique model of a Datalog program), and program $\Pi \cup \{p \leftarrow A\}$ can be built using constant space, the complexity result is obtained.

\begin{restatable}{theorem}{ThmHardTwo}\label{thm:hard:2}
G-coherence testing is P-hard for ASP(M).
\end{restatable}

For the disjunctive case, instead, the following transformation is used.

\begin{definition}\label{def:m-rewriting}
Let $\Pi$ be in ASP($\neg$, $\vee$).
Let $M(\Pi)$ be the program obtained from $\Pi$ by replacing every occurrence of a negative literal $\naf p$ with $p^F$, where $p^F$ is a fresh propositional atom associated with $p$, and by adding rule
$p \vee p^F \leftarrow A$,
where $\mathit{dom}(A) = \{p\}$ and $A(I) = |\{p\} \cap I| \geq 0$, for all $I \subseteq \A$.
\end{definition}

\begin{restatable}{lemma}{LemMrewriting}\label{lem:m-rewriting}
Let $\Pi$ be in ASP($\neg$, $\vee$).
The following relation holds:
$\mathit{GSM}(\Pi) \equiv_{At(\Pi)} \mathit{GSM}(M(\Pi))$.
\end{restatable}

Since $M(\Pi)$ is polynomial-time constructible and only comprises monotone aggregates, $\Sigma^P_2$-hardness follows.

\begin{restatable}{theorem}{ThmHardThree}\label{thm:hard:3}
G-coherence testing is $\Sigma^P_2$-hard for ASP($\vee$, M).
\end{restatable}

\subsection{Complexity of Cautious Reasoning}

As in the previous section, the first row of Table~\ref{tab:complexity} reports well-known results concerning complexity of cautious reasoning for programs without aggregates.
Moreover, membership in $\Pi^P_2$ is proved in \cite{DBLP:journals/tplp/GelfondZ14} for the general case.
For a program $\Pi$ in ASP(M), membership in P is obtained by noting that the unique G-stable model candidate of $\Pi$ can be computed in polynomial-time, as shown in the previous section.

\begin{restatable}{theorem}{ThmGcaP}\label{thm:gca-p}
G-cautious reasoning is in P for ASP(M).
\end{restatable}

For a disjunction-free program $\Pi$ and a propositional atom $p$, the complementary problem can be solved by guessing an interpretation $I$ such that $p \notin I$, and checking that $I$ is a G-stable model of $\Pi$.
It is a polytime check because $G(\Pi,I)$ is ASP($-$).

\begin{restatable}{theorem}{ThmGcaConp}\label{thm:gca-conp}
G-cautious reasoning is in co-NP for programs in ASP($\neg$, M, C, N).
\end{restatable}

As for the hardness, first of all observe that P-hardness for the simplest case provides P-hardness for any case.
Moreover, coherence testing can be reduced to (the complement of) cautious reasoning in general.
In more detail, a program $\Pi$ is G-coherent if and only if $\bot$ is not a G-cautious consequence of $\Pi$, where $\bot \in \A \setminus At(\Pi)$ is an atom not occurring in $\Pi$.
In fact, if $\Pi$ is G-coherent then its G-stable models cannot contain $\bot$, and therefore $\bot$ is not a G-cautious consequence of $\Pi$.
Otherwise, $\bot$ is a G-cautious consequence because $\Pi$ has no stable models.
Since G-coherence of $\Pi$ can be equivalently checked on $M(\Pi)$, $C(\Pi)$, or $N(\Pi)$, all other hardness results are obtained from Theorems~\ref{thm:hard:1}--\ref{thm:hard:3}.

\begin{restatable}{theorem}{ThmGcoHard}\label{thm:gco-hard}
G-cautious consequence is $\Pi^P_2$-hard for \mbox{ASP($\vee$, M)}, ASP($\vee$, C) and ASP($\vee$, N).
It is co-NP-hard for ASP(C) and ASP(N).
\end{restatable}

\section{Compilation}\label{sec:compilation}

G-stable models of a logic program can be computed by compiling into F-stable model semantics, for which efficient implementation are available.
(Translations into other frameworks, for example Ferraris's propositional theories \cite{DBLP:conf/lpnmr/Ferraris05a}, are also possible, but not the focus of this paper.)
Two different rewritings are presented in this section.
The first rewriting is more compact, in the sense that it introduces fewer auxiliary atoms.
The second rewriting instead requires more auxiliary atoms, but has the advantage that the output program only comprises stratified aggregates (essentially, in these programs no recursive definition involves an aggregate; see \cite{DBLP:journals/ai/FaberPL11} for a formal definition).

\begin{definition}[Rewriting 1]\label{def:rew}
Let $\Pi$ be a program.
Let $\mathit{rew}(\Pi)$ be the program obtained from $\Pi$ by performing the following operations:
\begin{enumerate}
\item For each $p \in At(\Pi)$, a fresh propositional atom $p'$ and the following rules are introduced:
$p' \leftarrow \naf p$, and $p' \leftarrow p$.

\item
For each aggregate $A$ occurring in a rule $r$ of $\Pi$ and such that $\mathit{dom}(A) = \{p_1,\ldots,p_n\}$ (for some $n \geq 0$), literals $p_1',\ldots,p_n'$ are added to the body of $r$.
\end{enumerate}
\end{definition}

\begin{example}\label{ex:rew}
Consider again program $\Pi_1$ from Example~\ref{ex:syntax}, whose G-stable models are $\emptyset$ and $\{a,c\}$. 
Program $\mathit{rew}(\Pi_1)$ is the following:
$\{a \leftarrow \naf\naf a;$
$b \vee c \leftarrow A_1, a', b'\} \cup \{a' \leftarrow \naf a;$
$a' \leftarrow a;$
$b' \leftarrow \naf b;$
$b' \leftarrow b;$
$c' \leftarrow \naf c$;
$c' \leftarrow c\}$.
Its F-stable models are the following:
$\emptyset \cup X$ and $\{a,c\} \cup X$, where $X = \{a',b',c'\}$.
In fact, $a',b',c'$ are necessarily true because of the rules introduced in item 1 of Definition~\ref{def:rew}.
Moreover, note that if $a$ is false in some model $I$ then $a'$ is necessarily true in any model of the reduct $F(\Pi_1,I)$.
On the other hand, if $a$ is true in $I$ then $a'$ can be possibly assumed false in a model of $F(\Pi_1,I)$.
Similarly for $b$ and $b'$, and for $c$ and $c'$.
\hfill $\blacksquare$
\end{example}

Intuitively, as also observed in the above example, all auxiliary propositional atoms are true in any model $I$ of $\mathit{rew}(\Pi)$ because of the rules introduced in item 1 of Definition~\ref{def:rew}.
Moreover, if $J$ is a minimal model of $F(\mathit{rew}(\Pi),I)$ then the following properties are satisfied:
(i) if $p \notin I$, then $p' \in J$;
(ii) if $p \in I$, then $p' \in J$ if and only if $p \in J$.
Correctness of the first compilation is thus established.

\begin{restatable}{theorem}{ThmEqRew}\label{thm:eq-rew}
Let $\Pi$ be a program.
The following relation holds:
$\mathit{GSM}(\Pi) \equiv_{At(\Pi)} \mathit{FSM}(\mathit{rew}(\Pi))$.
\end{restatable}

A drawback of this first compilation is that the evaluation of the resulting program may be on a higher complexity class than the evaluation of the original program.
For example, G-coherence testing of disjunction-free programs is NP-complete in general, while a $\Sigma^P_2$ procedure will be used to test F-coherence of the rewritten program.
Such a drawback motivates the introduction of a second compilation.
To ease the presentation, and to provide a better analysis later, the syntax of the language is extended by allowing the use of \emph{integrity constraints}, that is, rules of the form (\ref{eq:rule}) with empty heads.
Note that the semantics provided in Section~\ref{sec:background} can already cope with such an extension.

\begin{definition}[Rewriting 2]\label{def:str}
Let $\Pi$ be a program.
Let $\mathit{str}(\Pi)$ be the (stratified) program obtained from $\Pi$ by performing the following operations:
\begin{enumerate}
\item For each $p \in At(\Pi)$, two fresh propositional atoms $p',p''$ and the following three groups of rules are introduced:
(i) $p' \leftarrow \naf p$, and $p' \leftarrow p$;
(ii) $p'' \leftarrow \naf\naf p''$;
(iii) $\leftarrow \naf p'', p$, and $\leftarrow p'', \naf p$.

\item
For each aggregate $A$ occurring in a rule $r$ of $\Pi$ and such that $\mathit{dom}(A) = \{p_1,\ldots,p_n\}$ ($n \geq 0$), literals $p_1',\ldots,p_n'$ are added to $B(r)$,
and $A$ is replaced by a new aggregate $A''$ such that $\mathit{dom}(A'') = \{p_1'',\ldots,p_n''\}$ and $A''(I) = A(\{p \in \A \mid p'' \in I\})$, for all $I \subseteq \A$.
\end{enumerate}
\end{definition}

\begin{example}\label{ex:str}
Resorting again to $\Pi_1$ from Example~\ref{ex:syntax}, $\mathit{str}(\Pi_1)$ is the following program:
$\{a \leftarrow \naf\naf a;$
$b \vee c \leftarrow A_1'', a', b'\} \cup \{a' \leftarrow \naf a;$
$b' \leftarrow \naf b;$
$c' \leftarrow \naf c;$
$a' \leftarrow a;$
$b' \leftarrow b;$
$c' \leftarrow c\} \cup \{a'' \leftarrow \naf\naf a'';$
$b'' \leftarrow \naf\naf b'';$
$c'' \leftarrow \naf\naf c''\} \cup \{\leftarrow \naf a'', a;$
$\leftarrow \naf b'', b;$
$\leftarrow \naf c'', c;$
$\leftarrow a'', \naf a;$
$\leftarrow b'', \naf b;$
$\leftarrow c'', \naf c\}$.
where $\mathit{dom}(A_1'') = \{a'',b''\}$ and $A_1''(I) = |\{a'',b''\} \cap I| \geq 1$, for all $I \subseteq \A$.
The F-stable models of $\mathit{str}(\Pi_1)$ are the following:
$\emptyset \cup X$ and $\{a,c\} \cup X \cup \{a'',c''\}$, where $X = \{a',b',c'\}$.
In fact, for atoms $a',b',c'$, comments in Example~\ref{ex:rew} apply.
Atom $a''$ instead is forced to have the same truth value of $a$ because of rules of the group (iii).
Similarly for $b''$ and $b$, and for $c''$ and $c$.
Moreover, atoms $a'',b'',c''$ fix the interpretation of $A_1''$ in the reduct thanks to rules of the group (ii).
\hfill $\blacksquare$
\end{example}

Rules of the group (i) are as in the first compilation, and therefore the already discussed properties on atoms of the form $p'$ hold for $\mathit{str}(\Pi)$.
Rules of the group (ii), instead, are used to guess an interpretation for atoms of the form $p''$.
Actually, they also force the interpretation of atoms of the form $p$ because of rules of the group (iii).
However, while the interpretation of an atom $p''$ is fixed also in the reduct, the interpretation of an atom $p$ can be changed.
Also the interpretation of aggregates is fixed in the reduct because their domains only contain atoms of the form $p''$.
Correctness of the second compilation is finally established.

\begin{restatable}{theorem}{ThmEqStr}\label{thm:eq-str}
Let $\Pi$ be a program.
The following relation holds:
$\mathit{GSM}(\Pi) \equiv_{At(\Pi)} \mathit{FSM}(\mathit{str}(\Pi))$.
\end{restatable}

\subsection{Properties}

The rewritings introduced in the previous section are \emph{polynomial, faithful and modular} translation functions \cite{DBLP:journals/jancl/Janhunen06}, i.e., they are polynomial-time computable, preserve stable models (if auxiliary atoms are ignored), and can be computed independently on parts of the input program.
In fact, faithfulness is preserved because of Theorems~\ref{thm:eq-rew}--\ref{thm:eq-str}, and modularity can be easily proved by assuming that different auxiliary atoms are used for different parts of the program.

\begin{restatable}{theorem}{ThmModular}\label{thm:modular}
Let $\Pi,\Pi'$ be programs such that $\Pi \cap \Pi' = \emptyset$.
For $tr \in \{\mathit{rew}, \mathit{str}\}$, the following conditions are satisfied:
$tr(\Pi \cup \Pi') = tr(\Pi) \cup tr(\Pi')$, and
$tr(\Pi) \cap tr(\Pi') = \emptyset$.
\end{restatable}

It is also possible to show that the rewritings are polynomial-time computable and have linear size with respect to the original program.
For this purpose, the size of a program $\Pi$, denoted $\norm{\Pi}$, is defined as the number of symbols occurring in $\Pi$.
In more detail, every occurrence of a propositional atom or of a negated literal is considered one symbol, while every occurrence of an aggregate $A$ is counted as $|\mathit{dom}(A)|$ symbols.
(No other symbol is considered in the size of $\Pi$.)
The rewriting in Definition~\ref{def:rew} introduces $|At(\Pi)|$ new propositional atoms, and $2 \cdot |At(\Pi)|$ new rules of size 2.
Moreover, for each rule $r$ in $\Pi$, program $\mathit{rew}(\Pi)$ contains a rule of size at most $2 \cdot \norm{r}$ (because for each aggregate $A$ in $r$, $|\mathit{dom}(A)|$ propositional atoms are added to the body of $r$).
The rewriting in Definition~\ref{def:str}, instead, introduces $2 \cdot |At(\Pi)|$ new propositional atoms, and $5 \cdot |At(\Pi)|$ new rules of size 2.
The other rules in $\mathit{str}(\Pi)$ are obtained from rules in $\Pi$ and have the same size of the corresponding rules in $\mathit{rew}(\Pi)$.

\begin{restatable}{theorem}{ThmSizeRew}\label{thm:size-rew}\label{thm:size-str}
Let $\Pi$ be a program.
The programs $\mathit{rew}(\Pi)$ and $\mathit{str}(\Pi)$ are polynomial-time constructible, and the following relations holds:
(i) $\norm{\mathit{rew}(\Pi)} \leq 4 \cdot |At(\Pi)| + 2 \cdot \norm{\Pi}$;
(ii) $\norm{\mathit{str}(\Pi)} \leq 10 \cdot |At(\Pi)| + 2 \cdot \norm{\Pi}$.
\end{restatable}

There are a few additional observations concerning the rewritings presented in the previous section, which also positively affect their sizes.
The first observation is that fresh atoms could be added just for propositional atoms belonging to the domain of some aggregate occurring in $\Pi$.
In fact, note that atoms $c',c''$ are not required in the rewritings reported in Examples~\ref{ex:rew}--\ref{ex:str}.
Such atoms are included in Definitions~\ref{def:rew}--\ref{def:str} to simplify the presentation of the rewritings.
The second observation is more technical and concerns the implementations of current ASP solvers, which are essentially based on $F$-stable model semantics.
ASP solvers use two modules, called \emph{model generator} and \emph{model checker}.
The first module produces a model $I$ of the input program $\Pi$, while the second module tests the stability of $I$, i.e., it checks whether no strict subset of $I$ is a model of $F(\Pi,I)$.
In both rewritings, atoms of the form $p'$ are irrelevant for the model generator, in the sense that they are immediately derived true.
Hence, the search space of the model generator is not increased at all when $\mathit{rew}(\Pi)$ is processed.
A similar observation also applies to $\mathit{str}(\Pi)$.
Indeed, atoms of the form $p''$ are constrained to have the same truth values of the corresponding atoms of the form $p$ because of rules of the group (3).
In addition, atoms of the form $p''$ are irrelevant for the model checker because their interpretation is fixed in this module by rules of the group (2).
As a consequence, also the interpretation of all aggregates in $\mathit{str}(\Pi)$ is fixed in the model checker because their domains only comprise atoms of the form $p''$.

\begin{restatable}{theorem}{ThmSearchSpace}\label{thm:search-space}
Let $\Pi$ be a program, and $I$ be an interpretation.
If $I \models \mathit{rew}(\Pi)$ or $I \models \mathit{str}(\Pi)$ then $\{p' \mid p \in At(\Pi)\} \subseteq I$.
Moreover, for each $J \subseteq I$ such that $J \models F(\mathit{str}(\Pi),I)$, it holds that $\{p'' \mid p \in I\} \subseteq J$.
\end{restatable}

A final observation, which is eventually linked to the previous, is that current ASP solvers only rely on the model generator to process disjunction-free programs.
More specifically, this is the case if additionally non-convex aggregates are stratified.
As already observed, the rewriting in Definition~\ref{def:str} is such that all aggregates in the rewritten program are stratified.
Moreover, note that $\mathit{str}(\Pi)$ does not introduce disjunction in $\Pi$ (this is also true for $\mathit{rew}(\Pi)$).

\begin{restatable}{theorem}{ThmStratified}\label{thm:stratified}
Let $\Pi$ be a program.
All aggregates in $\mathit{str}(\Pi)$ are stratified, and if $\Pi$ has no disjunction then both $\mathit{rew}(\Pi)$ and $\mathit{str}(\Pi)$ have no disjunction.
\end{restatable}

Therefore, checking G-coherence of $\Pi$ by means of checking F-coherence of $\mathit{str}(\Pi)$ is an appropriate technique from the complexity point of view, with the only corner case of ASP(M), i.e., programs without negation and disjunction, and whose aggregates are all monotone.
A similar comment applies to performing G-cautious reasoning on $\Pi$ by means of F-cautious reasoning on $\mathit{str}(\Pi)$.
In fact, in the disjunction-free case, the rewriting in Definition~\ref{def:str} provides alternative proofs for NP-membership of G-coherence testing and co-NP-membership of G-cautious reasoning.

\section{Related Work}\label{sec:related}

The challenge of extending stable model semantics with aggregate constructs has been investigated quite intensively in the previous decade.
Among the many proposals, F-stable model semantics \cite{DBLP:journals/tocl/Ferraris11,DBLP:journals/ai/FaberPL11} is of particular interest as many ASP solvers are currently based on this semantics \cite{DBLP:journals/ai/GebserKS12,DBLP:journals/tplp/FaberPLDI08}.
Actually, the definition provided in Section~\ref{sec:background} is slightly different than those in \cite{DBLP:journals/tocl/Ferraris11,DBLP:journals/ai/FaberPL11}.
In particular, the language considered in \cite{DBLP:journals/tocl/Ferraris11} has a broader syntax allowing for arbitrary nesting of propositional formulas.
The language considered in \cite{DBLP:journals/ai/FaberPL11}, instead, does not allow explicitly the use of double negation, which however can be simulated by means of auxiliary atoms.
For example, in \cite{DBLP:journals/ai/FaberPL11} a rule $p \leftarrow \naf\naf p$ must be modeled by using a fresh atom $p^F$ and the following subprogram:
$\{p \leftarrow \naf p^F;$ $p^F \leftarrow \naf p\}$.
On the other hand, negated aggregates are permitted in \cite{DBLP:journals/ai/FaberPL11}, while they are forbidden in this paper.
Actually, programs with negated aggregates are those for which \cite{DBLP:journals/tocl/Ferraris11} and \cite{DBLP:journals/ai/FaberPL11} disagree.
As a final remark, the reduct of \cite{DBLP:journals/ai/FaberPL11} does not remove negated literals from bodies, which however are necessarily true in all counter-models because double negation is not allowed in the syntax considered by \cite{DBLP:journals/ai/FaberPL11}.

Other relevant stable model semantics for logic programs with aggregates are reported in \cite{DBLP:journals/tplp/PelovDB07,DBLP:journals/tplp/SonP07} for disjunction-free programs, recently extended to the disjunctive case in \cite{DBLP:journals/ai/ShenWEFRKD14}.
In these semantics the stability check is not given in terms of minimality of the model for the program reduct but obtained by means of a fixpoint operator similar to immediate consequence, and the following relation holds in general: stable models of \cite{DBLP:journals/ai/ShenWEFRKD14} are a selection of F-stable models, and they coincide up to ASP($\neg$,M,C), which is also the complexity boundary between the first and second level of the polynomial hierarchy for F-stable model semantics \cite{DBLP:conf/lpnmr/AlvianoF13}.
Finally, a more recent proposal is G-stable model semantics \cite{DBLP:journals/tplp/GelfondZ14}, whose relation with other semantics has been highlighted by \cite{DBLP:conf/ijcai/AlvianoF15} in the disjunction-free case:
G-stable models are F-stable models, but the converse is not always true.

A detailed complexity analysis for F-stable models is reported in \cite{DBLP:journals/ai/FaberPL11} and summarized in Table~\ref{tab:complexity}.
Complexity of reasoning under stable models by \cite{DBLP:journals/tplp/PelovDB07,DBLP:journals/tplp/SonP07}, instead, is analyzed in \cite{pelo-2004}, where in particular $\Sigma^P_2$-completeness of coherence testing is proved for disjunction-free programs with aggregates.
Concerning G-stable models, the general case was studied in \cite{DBLP:journals/tplp/GelfondZ14}, and a more detailed analysis is provided by this paper.
In particular, for disjunction-free programs, the main reasoning tasks are in the first level of the polynomial hierarchy in general when G-stable models are used.
On the other hand, coherence testing jumps from $\mathbf{K}$ to $\Sigma^P_2$ when F-stable models are replaced by G-stable models in programs with monotone aggregates only.
Indeed, in constrast to previous semantics, monotone aggregates are enough to simulate integrity constraints and negation when G-stable models are used.

Techniques to rewrite logic programs with aggregates into equivalent aggregate-free programs were also investigated in the literature.
For example, a rewriting into aggregate-free programs is presented by \cite{DBLP:journals/tocl/Ferraris11} for F-stable model semantics.
However, it must be noted that the rewriting of \cite{DBLP:journals/tocl/Ferraris11} produces nested expressions in general, and current mainstream ASP systems cannot process directly such constructs, but instead require additional translations such as those by \cite{DBLP:conf/lpnmr/LeeP09}.
Other relevant rewriting techniques were proposed in \cite{DBLP:conf/lpnmr/BomansonJ13,DBLP:conf/jelia/BomansonGJ14}, also proved to be quite efficient in practice.
However, these rewritings preserve F-stable models only in the stratified case, or if recursion is limited to convex aggregates.

Aggregate functions are also semantically similar to DL \cite{DBLP:journals/ai/EiterILST08} and HEX atoms \cite{DBLP:journals/jair/EiterFKRS14}, extensions of ASP for interacting with external knowledge bases, possibly expressed in different languages.

\section{Conclusion}\label{sec:conclusion}

G-stable models are a recent proposal for interpreting logic programs with aggregates.
A detailed complexity analysis of the main reasoning tasks for this new semantics was reported in Section~\ref{sec:complexity}, highlighting similarities and differences versus mainstream ASP semantics, here referred to as F-stable models.
In more detail, G-coherence testing is NP-complete for disjunction-free programs, in contrast to $\Sigma^P_2$-completeness of F-coherence testing.
An even more surprising result was shown for negation-free programs with monotone aggregates:
Such programs are guaranteed to be F-coherent, 
while G-coherence testing was shown to be $\Sigma^P_2$-hard because negation can be simulated by means of disjunction and monotone aggregates in the new semantics. 
Similar results were shown for G-cautious reasoning.

A further link between G- and F-stable models is provided by the rewritings in Section~\ref{sec:compilation}:
G-stable models of an input program can be obtained by computing F-stable models of a rewritten program, where the size of the rewritten program is linear with respect to the size of the original program.
In particular, two different rewritings are presented and analyzed.
Moreover, one of these rewritings outputs programs with stratified aggregates only, which are handled efficiently by modern ASP solvers.
A prototype system supporting common aggregation functions such as \COUNT, \SUM, \AVG, \MIN, \MAX, \ODD, and \EVEN is thus implemented by means of this rewriting, and using the ASP solver \textsc{wasp} \cite{DBLP:conf/lpnmr/AlvianoDFLR13,DBLP:journals/tplp/AlvianoDR14} to obtain G-stable models of the original program.
%

\section*{Acknowledgement}
This work was partially supported by MIUR within project ``SI-LAB BA2KNOW  -- Business Analitycs to Know'', and by Regione Calabria, POR Calabria FESR 2007-2013, within project ``ITravel PLUS'' and project ``KnowRex''. 
Mario Alviano was partly supported by the National Group for Scientific Computation (GNCS-INDAM), and by Finanziamento Giovani Ricercatori UNICAL.

\bibliographystyle{acmtrans}
\bibliography{bibtex}

\clearpage
\appendix

\section{Proofs of Section~\ref{sec:complexity}}

\LemTpPoly*
\begin{proof}
We first show that the least fixpoint of $T_\Pi$ is polytime computable.
Let $\Pi$ be a program in ASP(M), and $I$ be an interpretation.
Computing $T_\Pi(I)$ requires to iterate over every rule $r$ of $\Pi$ and check whether $I \models B(r)$.
Checking $I \models B(r)$ can be done in polynomial-time if aggregates are polynomial-time computable functions, as it is assumed in this section.
Hence, a single application of $T_\Pi$ is polynomial-time computable.
The least fixpoint of $T_\Pi$ is computed, by definition, starting from $\emptyset$ and repeatedly applying $T_\Pi$.
Define $I_0 = \emptyset$, $I_{i+1} = T_\Pi(I_i)$ (for $i \geq 0$).
For each $i \geq 0$, either $I_{i+1} \setminus I_i \neq \emptyset$ or $I_i$ is the least fixpoint of $T_\Pi$.
Since atoms in $I_{i+1} \setminus I_i$ are among those in $At(\Pi)$, we have that $I_{|At(\Pi)|} = I_{|At(\Pi)|+1}$.

We now show the second part of the lemma.
$I \models \Pi$ by construction.
Note that $G(\Pi,I)$ is a plain Datalog program.
It is unique minimal model is the least fixpoint of $T_{G(\Pi,I)}$, i.e., interpretation $J$.
Hence, $I \in \mathit{GSM}(\Pi)$ if and only if $I = J$.
To complete the proof is enough to show that no other interpretation is a G-stable model of $\Pi$.
Let $K$ be an interpretation such that $K \neq I$ and $K \models \Pi$.
Therefore, $K \supset I$ because $I$ is the least fixpoint of $T_\Pi$.
To prove that $K \notin \mathit{GSM}(\Pi)$ note that $I \models G(\Pi,K)$.
\end{proof}

\ThmGcoDefP*
\begin{proof}
Let $I$ be the least fixpoint of $T_\Pi$.
$I$ is computable in polynomial-time because of Lemma~\ref{lem:tp-poly}.
Actually, $I$ is the only candidate to be a G-stable model of $\Pi$ because of Lemma~\ref{lem:tp-sem}.
To check whether $I \in \mathit{GSM}(\Pi)$, build $G(\Pi,I)$ and compute the least fixpoint of $T_{G(\Pi,I)}$, again in polynomial-time because of Lemma~\ref{lem:tp-poly}.
If the two least fixpoints are equal then $\Pi$ is G-coherent, otherwise it is G-incoherent.
\end{proof}

\ThmGcoNormalNp*
\begin{proof}
Let $\Pi$ be in ASP($\neg$, M, C, N), and $I$ be an interpretation.
We provide a polynomial-time procedure for checking that $I$ is a G-stable model of $\Pi$.
The procedure first checks that $I \models \Pi$ in polynomial-time.
If it is the case, the procedure builds the reduct $G(\Pi,I)$, again in polynomial-time.
Program $G(\Pi,I)$ is in ASP($-$) and therefore Lemma~\ref{lem:tp-poly} can be applied to obtain the unique minimal model of $G(\Pi,I)$, say $J$, in polynomial-time.
If $I = J$ then the procedure accepts $I$ as a G-stable model, otherwise it rejects $I$.
\end{proof}

\LemSimpleRewritings*
\begin{proof}
Let $I$ be an interpretation.
$I \models \Pi$ if and only if $I \models C(\Pi)$.
In particular, if $\naf p$ is replaced by an aggregate $A$ in a rule $r$, we have $I \models \naf p$ if and only if $I \models A$.
Note that $I \not\models \naf p$ implies that $r$ is removed in the reducts $G(\Pi,I),G(C(\Pi),I)$, while $I \models \naf p$ implies that both $\naf p$ and $A$ are replaced by the empty set in the rules obtained from $r$ in the reducts.
We therefore conclude that $G(\Pi,I) = G(\Pi,C(I))$, from which we obtain $\mathit{GSM}(\Pi) \equiv_{At(\Pi)} \mathit{GSM}(C(\Pi))$.

The proof of $\mathit{GSM}(\Pi) \equiv_{At(\Pi)} \mathit{GSM}(N(\Pi))$ is similar.
We have just to additionally note that $\bot \notin I$ holds for every $I \in \mathit{GSM}(\Pi) \cup \mathit{GSM}(N(\Pi))$.
\end{proof}

\ThmHardOne*
\begin{proof}
G-coherence testing is $\Sigma^P_2$-hard for ASP($\neg$, $\vee$), and it is NP-hard for ASP($\neg$) \cite{DBLP:journals/amai/EiterG95}.
G-coherence of $\Pi$ can be reduced to G-coherence testing of $C(\Pi)$ or of $N(\Pi)$ because of Lemma~\ref{lem:simple-rewritings}.
Since $C(\Pi)$ and $N(\Pi)$ can be computed in polynomial-time, do not introduce disjunction, eliminate negation, and only have convex and non-convex aggregates, respectively, the proof is complete.
\end{proof}

\ThmHardTwo*
\begin{proof}
G-cautious reasoning over Datalog programs is P-hard \cite{DBLP:journals/amai/EiterG95}.
We reduce this problem to G-coherence testing of disjunction- and negation-free programs with monotone aggregates.
Let $\Pi$ be in ASP($-$), and $p$ be a propositional atom.
Program $\Pi' = \Pi \cup \{p \leftarrow A\}$, where $\mathit{dom}(A) = \{p\}$ and $A(I) = |\{p\} \cap I| \geq 0$, can be built using only logarithmic space.
Since $\Pi$ is a Datalog program, it has a unique G-stable model, say $I$.
If $p \in I$ then $p$ belongs to the least fixpoint of $T_{\Pi}$ because of Lemma~\ref{lem:tp-sem}, and therefore it belongs to the least fixpoint of $T_{\Pi'}$ too because of monotonicity.
On the other hand, if $p \notin I$ then any model $J$ of $\Pi'$ is such that $J \supset I$ because of rule $p \leftarrow A$ (note that $A$ is always true).
We conclude that $G(\Pi',J) = G(\Pi,J) \cup \{p \leftarrow p\}$, and therefore the least fixpoint of $T_{G(\Pi',J)}$, which is equal to the least fixpoint of $T_{G(\Pi,J)}$, is a subset of $I$.
We conclude that $J$ is not a G-stable model of $\Pi'$ and hence $\Pi'$ is G-incoherent.
\end{proof}

\LemMrewriting*
\begin{proof}
Without loss of generality, let us assume that all atoms in $At(\Pi)$ occur negated in $\Pi$ at least once.
Let $I$ be a G-stable model of $\Pi$.
Define $I^F = I \cup \{p^F \mid p \notin I\}$.
We have $I^F \models M(\Pi)$.
Concerning $G(M(\Pi),I^F)$ note that for each $p \in At(\Pi)$ rule $p \vee p^F \leftarrow A$ is either replaced by
\[
p \vee p^F \leftarrow
\]
in case $p \notin I$, or by
\[
p \vee p^F \leftarrow p
\]
if $p \in I$.
In the first case, the rule guarantees that every model $J$ of $G(M(\Pi),I^F)$ such that $J \subseteq I$ satisfies $p^F \in J$.
Hence, rules of $G(M(\Pi),I^F)$ containing $p^F$ can be simplified by removing $p^F$, which essentially results into $G(\Pi,I)$ (plus rules obtained from $p \vee p^F \leftarrow A$).
In the second case, the rule is trivially satisfied by all interpretations, and therefore it can be removed from $G(M(\Pi),I^F)$.
Since $I$ is a minimal model of $G(\Pi,I^F)$, we have that $I^F$ is a minimal model of $G(M(\Pi),I^F)$, i.e., $I^F \in \mathit{GSM}(M(\Pi))$.

For the other direction, let $I$ be a G-stable model of $M(\Pi)$.
We shall show that $I \cap At(\Pi)$ is a G-stable model of $\Pi$.
First of all, note that $I \models A$ for any aggregate $A$ occurring in $M(\Pi)$, and therefore $I \cap \{p,p^F\} \neq \emptyset$ because of rule $p \vee p^F \leftarrow A$, for all $p \in At(\Pi)$.
Moreover, since $I$ is a minimal model of $G(M(\Pi),I)$ by assumption, and $p^F$ does not occur in any other rule heads, we have $|I \cap \{p,p^F\}| = 1$.
We can therefore argument as in the previous direction and conclude that $I \cap At(\Pi)$ is a minimal model of $G(\Pi,I \cap At(\Pi))$, i.e., $I \cap At(\Pi) \in \mathit{GSM}(\Pi)$.

As a final observation, note that also $|\mathit{GSM}(\Pi)| = |\mathit{GSM}(M(\Pi))|$ holds because in any G-stable model of $M(\Pi)$ truth values for atoms of the form $p^F$ are implied by truth values of atoms of the form $p$.
\end{proof}

\ThmHardThree*
\begin{proof}
G-coherence testing is $\Sigma^P_2$-hard for a program $\Pi$ in ASP($\neg$, $\vee$) \cite{DBLP:journals/amai/EiterG95}.
G-coherence of $\Pi$ can be reduced to G-coherence testing of $M(\Pi)$ because of Lemma~\ref{lem:m-rewriting}.
Since $M(\Pi)$ can be computed in polynomial-time, eliminates negation, and only has monotone aggregates, the proof is complete.
\end{proof}

\ThmGcaP*
\begin{proof}
We provide a procedure for checking whether a given propositional atom $p$ is a G-cautious consequence of $\Pi$.
The procedure first checks G-coherence of $\Pi$ in polynomial-time (Theorem~\ref{thm:gco-def-p}).
If $\Pi$ is G-incoherent then the procedure rejects.
Otherwise, because of Lemma~\ref{lem:tp-sem}, the unique G-stable model of $\Pi$, say $I$, is the least fixpoint of $T_\Pi$.
The procedure then computes $I$ in polynomial-time (Lemma~\ref{lem:tp-poly}), and accepts if $p \in I$, otherwise it rejects.
\end{proof}

\ThmGcaConp*
\begin{proof}
Let $\Pi$ be in ASP($\neg$, M, C, N), and $p$ a propositional atom.
We prove that the complementary problem, checking the existence of a G-stable model $I$ of $\Pi$ such that $p \notin I$, is in NP.
To this aim, let $I$ be an interpretation such that $p \notin I$.
The following is a polynomial-time procedure for checking that $I$ is a G-stable model of $\Pi$:
The procedure first builds $G(\Pi,I)$, which is \mbox{disjunction-,} negation and aggregate-free.
Then, it computes the unique G-stable model, say $J$, of $G(\Pi,I)$, i.e., the least fixpoint of $T_{G(\Pi,I)}$ (Lemma~\ref{lem:tp-poly}), and accepts if $I = J$.
\end{proof}

\ThmGcoHard*
\begin{proof}
G-cautious reasoning is $\Pi^P_2$-hard for ASP($\neg$, $\vee$) already for programs in which negation only occurs in a rule of the form $w \leftarrow \naf w$ \cite{DBLP:journals/amai/EiterG95}.
Therefore, let us consider a program $\Pi = \Pi' \cup \{w \leftarrow \naf w\}$, where $\Pi'$ is in ASP($\vee$).
From Lemmas~\ref{lem:simple-rewritings}--\ref{lem:m-rewriting}, $\mathit{GSM}(\Pi) \equiv_{At(\Pi)} \mathit{GSM}(M(\Pi)) \equiv_{At(\Pi)} \mathit{GSM}(C(\Pi)) \equiv_{At(\Pi)} \mathit{GSM}(N(\Pi))$.
Let $p$ be a propositional atom among those in $At(\Pi)$.
It holds that $p$ is a G-cautious consequence of $\Pi$ if and only if $p$ is a G-cautious consequence of the other programs.
Hence, $\Pi^P_2$-hardness follows.

Similarly, G-cautious reasoning for ASP($\neg$) is co-NP-hard already for programs in which negation only occurs in a rule of the form $w \leftarrow \naf w$.
Since $C(\Pi)$ and $N(\Pi)$ are disjunction-free if $\Pi$ is disjunction-free, co-NP-hardness follows.
\end{proof}

\section{Proofs of Section~\ref{sec:compilation}}

\ThmEqRew*
\begin{proof}
Let $I$ be a G-stable model of $\Pi$.
We shall show that $I' = I \cup \{p' \mid p \in At(\Pi)\}$ is an F-stable model of $\mathit{rew}(\Pi)$.
In fact, $I' \models \mathit{rew}(\Pi)$ because $I \models \Pi$.
Consider a model $J \subseteq I$ of the reduct $F(\mathit{rew}(\Pi),I)$.
We have $J \cap At(\Pi) \models G(\Pi,I)$, and therefore $J \cap At(\Pi) = I$ holds because $I$ is a G-stable model of $\Pi$ by assumption.
Because of rules of introduced by item 1 in Definition~\ref{def:rew}, $J \cap At(\Pi) = I$ implies $J = I$, i.e., $I$ is an F-stable model of $\mathit{rew}(\Pi)$.

Let $I$ be an F-stable model of $\mathit{rew}(\Pi)$.
We shall show that $I \cap At(\Pi)$ is a G-stable model of $\Pi$.
First of all, note that $\{p' \mid p \in At(\Pi)\} \subseteq I$ because $I \models \Pi$ and because of rules introduced by item 1 in Definition~\ref{def:rew}.
Therefore, $I \cap At(\Pi) \models \Pi$ follows.
Consider a model $J \subseteq I \cap At(\Pi)$ of the reduct $G(\Pi,I)$.
We have $J \cup \{p' \mid p \in At(\Pi)\} \models F(\mathit{rew}(\Pi),I)$, and therefore $J \cup \{p' \mid p \in At(\Pi)\} = I$ because $I$ is an F-stable model of $\mathit{rew}(\Pi)$ by assumption.
It follows that $J = I \cap At(\Pi)$, i.e., $I \cap At(\Pi)$ is a G-stable model of $\Pi$.

Finally, note that also $|\mathit{GSM}(\Pi)| = |\mathit{FSM}(\mathit{rew}(\Pi))|$ holds because the mappings used above are one-to-one.
\end{proof}

\ThmEqStr*
\begin{proof}
Let $I$ be a G-stable model of $\Pi$.
We shall show that $I' = I \cup \{p' \mid p \in At(\Pi)\} \cup \{p'' \mid p \in I\}$ is an F-stable model of $\mathit{str}(\Pi)$.
In fact, $I' \models \mathit{str}(\Pi)$ because $I \models \Pi$.
Consider a model $J \subseteq I$ of the reduct $F(\mathit{str}(\Pi),I)$.
We have $J \cap At(\Pi) \models G(\Pi,I)$, and therefore $J \cap At(\Pi) = I$ holds because $I$ is a G-stable model of $\Pi$ by assumption.
Because of rules of the groups (i)--(ii) in Definition~\ref{def:str}, $J \cap At(\Pi) = I$ implies $J = I$, i.e., $I$ is an F-stable model of $\mathit{str}(\Pi)$.

Let $I$ be an F-stable model of $\mathit{str}(\Pi)$.
We shall show that $I \cap At(\Pi)$ is a G-stable model of $\Pi$.
First of all, note that $\{p' \mid p \in At(\Pi)\} \subseteq I$ because $I \models \Pi$ and because of rules of the group (i).
Moreover, note that $p \in I$ if and only if $p'' \in I$ because of rules of the group (iii), for all $p \in At(\Pi)$.
And also note that for each aggregate $A''$ occurring in $\mathit{str}(\Pi)$, $I \models A''$ if and only if $I \cap At(\Pi) \models A$.
Therefore, $I \cap At(\Pi) \models \Pi$ follows.
Consider a model $J \subseteq I \cap At(\Pi)$ of the reduct $G(\Pi,I)$, and define $J' = J \cup \{p' \mid p \in At(\Pi)\} \cup \{p'' \mid p \in I\}$.
We have $J' \models F(\mathit{str}(\Pi),I)$, and therefore $J' = I$ because $I$ is an F-stable model of $\mathit{str}(\Pi)$ by assumption.
It follows that $J = I \cap At(\Pi)$, i.e., $I \cap At(\Pi)$ is a G-stable model of $\Pi$.

Finally, note that also $|\mathit{GSM}(\Pi)| = |\mathit{FSM}(\mathit{str}(\Pi))|$ holds because the mappings used above are one-to-one.
\end{proof}

\ThmModular*
\begin{proof}
Immediate because the rewritings work on one rule at time.
\end{proof}

\ThmSizeRew*
\begin{proof}
We first prove relation (i).
Program $\mathit{rew}(\Pi)$ contains 2 rules for each atom in $At(\Pi)$, each one of size 2, and a rule for each rule of $\Pi$.
The number of atoms in these rules is at most two times the number of atoms in the original rules.

We now show relation (ii).
Program $\mathit{rew}(\Pi)$ contains 5 rules for each atom in $At(\Pi)$, each one of size 2, and a rule for each rule of $\Pi$.
The number of atoms in these rules is at most two times the number of atoms in the original rules.
\end{proof}

\ThmSearchSpace*
\begin{proof}[Proof of Theorem~\ref{thm:search-space}]
If $I$ satisfies rules introduced by item 1 in Definition~\ref{def:rew}, or equivalently of the group (i) in Definition~\ref{def:str}, then $\{p' \mid p \in At(\Pi)\} \subseteq I$.
Consider a model $J \subseteq I$ of the reduct $F(\mathit{str}(\Pi),I)$.
For each $p'' \in I$, $F(\mathit{str}(\Pi),I)$ contains a rule $p'' \leftarrow$ because of rules of the group (ii) in Definition~\ref{def:str}.
\end{proof}

\ThmStratified*
\begin{proof}
We first provide a more formal definition of stratified aggregate.
The \emph{dependency graph} of $\Pi$ has a node $p$ for each atom $p \in \mathit{At}(\Pi)$, and an arc from $q$ to $p$ if there is a rule $r \in \Pi$ such that $p \in H(r)$ and $q$ occurs in $B(r)$, either as a possibly negated literal or in the domain of an aggregate.
$\Pi$ is stratified with respect to aggregates if there is no rule $r \in \Pi$ such that $p \in H(r)$ and $q$ occurring in $B(r)$ belong to the same \emph{strongly connected component} of $\Pi$.

Let $\Pi$ be a program, and $A$ be an aggregate in $\mathit{str}(\Pi)$.
Hence, by construction, $\mathit{dom}(A) \subseteq \{p'' \mid p \in At(\Pi)\}$.
Note that all rules whose head contains some atom in $\mathit{dom}(A)$ belong to the group (ii) in Definition~\ref{def:str}, and therefore each atom $p'' \in \mathit{dom}(A)$ belongs to a singleton strongly connected component.
Stratification of aggregates in $\mathit{str}(\Pi)$ is thus proved.

Let $\Pi$ be a program without disjunction.
Program $\mathit{rew}(\Pi)$ and $\mathit{str}(\Pi)$ contain rules of the groups (i)--(iii), which have no disjunction, and rules obtained from those in $\Pi$ by replacing aggregates.
Hence, neither $\mathit{rew}(\Pi)$ nor $\mathit{str}(\Pi)$ has disjunction.
\end{proof}

\end{document}